\documentclass{article}
\usepackage{authblk}

\date{}

\PassOptionsToPackage{numbers, sort&compress}{natbib}
\usepackage[table]{xcolor}  
\definecolor{VibrantPink}{HTML}{FF4FB3} 
\PassOptionsToPackage{
    colorlinks=true,
    citecolor=VibrantPink,
    linkcolor=VibrantPink,
    urlcolor=VibrantPink
}{hyperref}

\usepackage[preprint]{neurips_2026}

\usepackage[utf8]{inputenc} 
\usepackage[T1]{fontenc}    
\usepackage{microtype}      
\usepackage[table]{xcolor}  
\usepackage{bbm}            
\usepackage{amsfonts}       
\usepackage{nicefrac}       
\usepackage{enumitem}       
\definecolor{LightBlue}{HTML}{1E90FF}
\usepackage{amsmath, amssymb, amsthm, mathtools}
\usepackage{cancel}

\theoremstyle{plain}
\newtheorem{theorem}{Theorem}[section]
\newtheorem{proposition}[theorem]{Proposition}

\theoremstyle{definition}

\theoremstyle{remark}

\usepackage{graphicx}
\usepackage{subcaption}
\usepackage{wrapfig}
\usepackage{multirow}
\usepackage{booktabs}   
\usepackage{longtable}
\usepackage{array}
\usepackage{caption}
\usepackage{algorithm}
\usepackage{algorithmic}
\usepackage[export]{adjustbox}

\definecolor{best}{RGB}{0,0,0}
\definecolor{codeblue}{rgb}{0.8,0.2,0.2}
\definecolor{codegray}{rgb}{0.5,0.5,0.5}
\definecolor{stdgray}{gray}{0.45}
\usepackage{xcolor}
\definecolor{brightorange}{RGB}{217,119,6}

\definecolor{ICMLAccent}{HTML}{2B6E73}
\definecolor{ICMLBg}{HTML}{F3F7F7}
\definecolor{ICMLStrong}{HTML}{1F4E79}
\definecolor{ICMLStrongBg}{HTML}{EEF4FA}

\usepackage{listings}
\usepackage[most]{tcolorbox}

\newtcolorbox{callout}[1][]{%
    breakable, enhanced, colback=ICMLBg, colframe=ICMLAccent,
    boxrule=0.5pt, arc=2mm, left=1.5mm, right=1.5mm, top=1.0mm, bottom=1.0mm,
    title={}, #1
}

\newtcolorbox{calloutimportant}[1][]{%
    breakable, enhanced, colback=ICMLStrongBg, colframe=ICMLStrong,
    boxrule=0.6pt, arc=2mm, left=1.8mm, right=1.5mm, top=1.0mm, bottom=1.0mm,
    borderline west={2.2pt}{0pt}{ICMLStrong}, title={}, #1
}

\newtcolorbox{calloutmotivation}[1][]{%
    breakable, enhanced, colback=ICMLStrongBg, colframe=ICMLStrong,
    boxrule=0.6pt, arc=2mm, left=1.8mm, right=1.5mm, top=1.0mm, bottom=1.0mm,
    borderline west={2.2pt}{0pt}{ICMLStrong}, title={}, fonttitle=\bfseries, #1
}

\usepackage{hyperref}
\hypersetup{
    colorlinks=true,
    citecolor=VibrantPink,
    linkcolor=VibrantPink,
    urlcolor=VibrantPink,
    pdftitle={Your Paper Title},
    pdfauthor={Your Name}
}

\usepackage[capitalize,noabbrev]{cleveref}

\newcommand{\pmt}[1]{{\scriptscriptstyle \pm #1}}

\title{Plan, Don’t Pose: Long Composite Motion Generation with Text-Aligned BFM}

\author[1,2]{Nikolay Shvetsov$^{*}$}
\author[3,4]{Maksim Bobrin}
\author[4,5]{Nazar Buzun}
\author[1]{Anton Bozhedarov}
\author[3,4]{Dmitry V. Dylov}

\affil[1]{AvaCapo, Potsdam, Germany}
\affil[2]{Potsdam University, Potsdam, Germany}
\affil[3]{Applied AI Institute, Computational Imaging Lab, Moscow, Russia}
\affil[4]{AXXX, Moscow, Russia}
\affil[5]{Innopolis University, Innopolis, Russia}

\begin{document}

\begingroup
\renewcommand\thefootnote{*}
\footnotetext{\texttt{n.shvetsov@avacapo.com}}
\endgroup

\maketitle
\begin{abstract}
Text-to-motion (T2M) generation has broad applications in character animation, virtual avatars, and human-robot interaction. Existing methods typically generate pose trajectories or motion tokens directly from language, forcing a single model to handle semantic interpretation, long-horizon structure, and low-level physical realization. This coupling makes them costly and often unreliable for long, compositional, or semantically dense prompts.
We propose Text2BFM, the first framework that aligns natural language with pretrained Behavioral Foundation Models (BFMs) for T2M generation without relying on heavy end-to-end motion generators. Text2BFM operates in the latent policy space of a frozen BFM, using it as an executable motion prior. A text-aligned variational behavioral bottleneck compresses BFM policy-latent sequences into compact motion representations that are compatible with language and preserve long-horizon behavioral structure.
Generation is performed in this compact behavioral manifold with a lightweight conditional generator, and the resulting latent encoded behaviors are decoded into policy latents that drive the pretrained frozen BFM. By decoupling semantic planning from motion execution, Text2BFM achieves efficient, robust T2M generation and strong performance on long, compositional textual descriptions.
\end{abstract}

\begin{figure}[htbp]
\centering
\footnotesize
\begin{tabular}{@{}c@{}}
\includegraphics[width=0.75\linewidth]{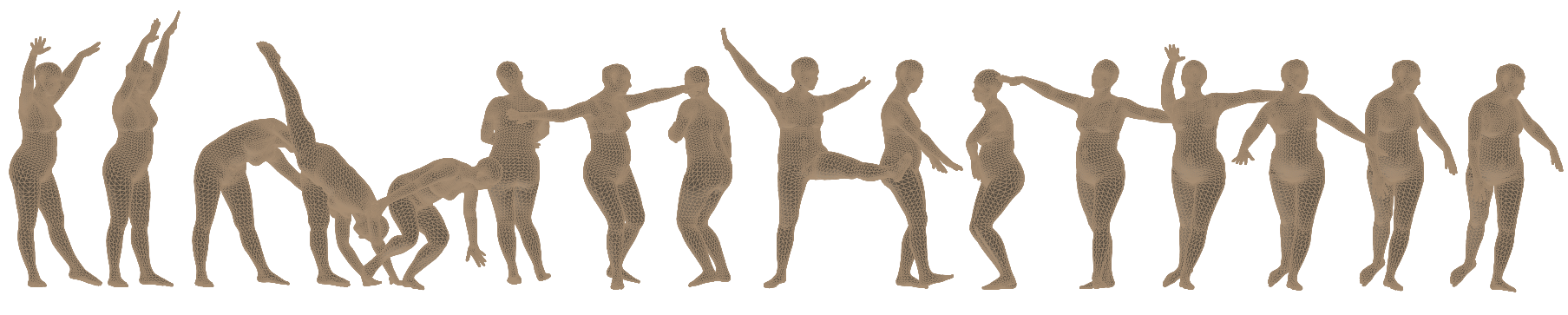} \\
\midrule
A person does a cartwheel, a spin, then two twirls, \\
and finishes with a low jump landing using ballet movements. \\
\includegraphics[width=0.75\linewidth]{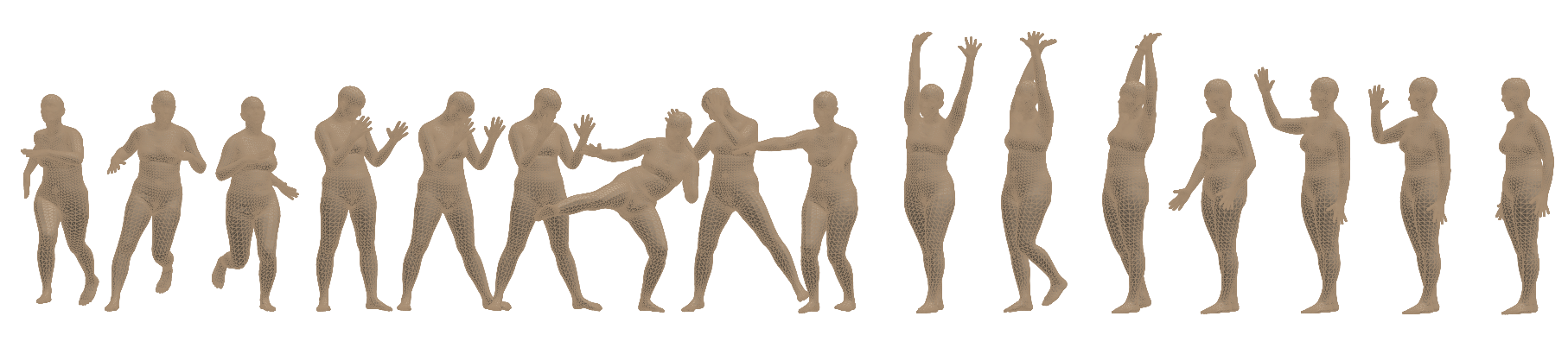} \\
\midrule
A person runs forward, then turns left, does a quick, sharp punch, and \\
follows it with a kick to the right. After that, he walks forward, \\
raises his arms in victory, and finishes by waving happily. \\
\end{tabular}
\caption{Two rollouts of proposed Text2BFM model, showcasing precise following of the instructions in the prompt while remaining semantically correct.}
\label{fig:text2bfm-rollouts}
\end{figure}

\section{Introduction}

Natural language describes human motion at the level of intentions, phases, and behaviors: a person walks toward an object, reaches for it, picks it up, and then performs another action. 
Most text-to-motion methods, however, generate pose trajectories or discrete motion tokens directly from language embeddings \citep{guo2022humanml3d, petrovich2022temos, zhang2022motiondiffuse, tevet2023mdm, chen2023mld, zhang2023t2mgpt, mo2023momask}. 
This requires a single generative model to simultaneously solve semantic interpretation, long-horizon temporal organization, and low-level physical realization. 
As a result, visually plausible motions can still exhibit foot sliding, unstable contacts, unnatural transitions, implausible joint configurations, or poor consistency under compositional prompts \citep{rempe2021humor, yuan2023physdiff, tessler2023calm}. 
We argue that this stems from a representational mismatch: language specifies behavior, while most text-to-motion models synthesize positions.

A more natural interface is to expose motion to the text model as executable behavior. 
\emph{Behavioral Foundation Models} (BFMs) provide such an interface by learning structured spaces of reusable behaviors from offline data \citep{metamotivo,sikchi2025rl,li2026bfmzero,bobrin2026zeroshot}. 
In particular, Forward-Backward representations \citep{blier2021learning, touati2021learning} enable promptable behavior inference through policy latents. 
A single global policy latent can represent short behaviors such as walking forward or turning, but long text prompts typically require temporally localized phases, transitions, and hierarchical structure \citep{petrovich2022temos, zhang2023t2mgpt, go2024gotozero, luo2024hierarchical}. 
Representing a motion as a trajectory of local policy latents $z_{1:T}$ is therefore more expressive, but a direct generation of such trajectories creates another long, weakly text-aligned sequence modeling problem.

We address this problem with a text-aligned variational behavioral bottleneck. 
Given a trajectory of Forward-Backward policy encoded representations $z_{1:T}$, the bottleneck compresses it into a more compact motion representations that can reconstruct executable local policy latents, while being aligned with textual descriptions.
Unlike a purely reconstructive autoencoder \citep{kingma2014auto, rezende2014stochastic}, our bottleneck is additionally regularized by a semantic-similarity objective that aligns the bottleneck with the corresponding text embedding, following the broader success of cross-modal representation learning \citep{radford2021learning, tevet2022motionclip, petrovich2023tmr}. 
The resulting plan space is both temporally compact and language-discriminative.

Text-conditioned generation is then performed in this compact behavioral plan space. 
Rather than generating poses or full policy-latent trajectories, a conditional generative model samples a text-consistent motion program, which is decoded into local policy latents and is executed by the pretrained FB policy. 
This separates semantic planning from physical realization: the compact program captures high-level temporal structure, while the behavioral policy provides executable motion. 
We instantiate the generator with Transformer-based flow matching, using flow matching as an efficient continuous generative model \citep{lipman2023flow, liu2023flow, albergo2023building} in the learned plan space rather than directly in the pose space \citep{peebles2023scalable, hy2024hymotion}.

Our contributions are:
\begin{itemize}
    \item We formulate text-to-motion generation as generation of compact executable behavioral programs, using policy-latent trajectories $z_{1:T}$ from a pretrained Forward-Backward behavioral model.
    \item We introduce a text-aligned variational behavioral bottleneck that compresses policy-latent trajectories while aligning the resulting motion programs with language.
    \item A unified framework for semantic planning and physically grounded motion execution via the learned bottleneck with text-conditioned flow matching and BFM-policy rollout.
    \item We show that the proposed framework improves semantic consistency and compositional ordering, while identifying a distributional-quality trade-off caused by the frozen BFM prior.
\end{itemize}

\section{Related Work}

\paragraph{Zero-Shot Reinforcement Learning.}
A central goal in reinforcement learning is to learn reusable behavioral structure that transfers to new tasks without retraining. Successor representations and successor features \citep{dayan1993improving,barreto2017successor} decompose value functions into dynamics- and reward-dependent components, enabling efficient transfer. Forward-Backward (FB) representations extend this idea by learning low-rank successor-measure factorizations and latent codes for executable behaviors \citep{touati2021learning,blier2021learning, touati2023does,agarwal2025proto}, making them well suited for motion generation as policy interfaces rather than descriptive embeddings.
However, standard FB policies usually rely on a single behavior latent or short-horizon command, which is too limited for multi-phase language prompts. We therefore represent motion as local policy latents $z_{1:T}$ and learn a compact, text-aligned variational bottleneck over them, preserving temporal structure while enabling generation in a lower-dimensional behavioral plan space. \cite{sikchi2025rl} also enables imitation learning via text descriptions, but relies on heavy image/video generation models, which limits applications.

\paragraph{Physics-Based Motion Control.}
Physics-based motion control methods learn executable motion priors from reference motion data. AMP regularizes control policies toward natural motion through adversarial imitation~\citep{peng2021amp}, while ASE learns reusable skill embeddings for downstream control~\citep{peng2022ase}. CALM learns directable latent controllers for virtual characters~\citep{tessler2023calm}, and PHC learns robust humanoid controllers for high-fidelity motion tracking and recovery~\citep{luo2023perpetual}. Unlike purely kinematic generators, these methods produce motion through policy execution, making them closely related to our use of BFM latents as executable behavioral commands.

\paragraph{Text-to-Motion Generation.}
Text-to-motion generation has progressed from recurrent and VAE-based models~\citep{ahuja2019language2pose, ghosh2021synthesis, petrovich2022temos} to diffusion-based, token-based, and latent generative approaches trained on large-scale datasets such as HumanML3D~\citep{guo2022humanml3d}. MotionDiffuse and MDM synthesize motion through conditional denoising~\citep{zhang2022motiondiffuse, tevet2023mdm}, while MLD performs diffusion in a learned motion latent space~\citep{chen2023mld}. T2M-GPT and MoMask instead rely on discrete motion tokens with autoregressive or masked generation~\citep{zhang2023t2mgpt, mo2023momask}.

Although these methods achieve strong visual quality, they primarily operate in pose space, motion-token space, or kinematic latent spaces. As a result, a single generator must handle semantic interpretation, temporal planning, and low-level physical realization, which can lead to weak long-horizon coherence, unstable contacts, foot sliding, or missing stages in compositional prompts. In contrast, our method generates text-conditioned behavioral programs that are decoded into executable policy latents, delegating low-level motion realization to a pretrained behavioral policy.

\paragraph{Language-Aligned Latent Motion Generation.}
Cross-modal representation learning has shown that language can semantically organize visual and motion embeddings, as in CLIP-style learning~\citep{radford2021learning}, MotionCLIP~\citep{tevet2022motionclip}, and TMR~\citep{petrovich2023tmr}. Our bottleneck similarly aligns motion programs with text, but keeps the latent space tied to executable policy commands rather than using it only for retrieval or conditioning. Latent generative models improve efficiency by operating on compact representations instead of high-dimensional sequences. VAEs provide a standard framework for stochastic compression~\citep{kingma2014auto, rezende2014stochastic}, while flow matching offers an efficient alternative to diffusion for continuous generation~\citep{lipman2023flow, liu2023flow, albergo2023building}. We use flow matching to generate compact text-conditioned behavioral programs in the learned bottleneck space, rather than poses directly.

Our work combines zero-shot behavioral representations, text-motion alignment, and latent flow-based generation. Unlike prior text-to-motion methods that synthesize poses or long motion-token sequences, our model generates compact, text-aligned behavioral programs that decode into policy latents and are realized through policy rollout.

\begin{figure*}[t]
    \centering
    \includegraphics[width=1.0\textwidth]{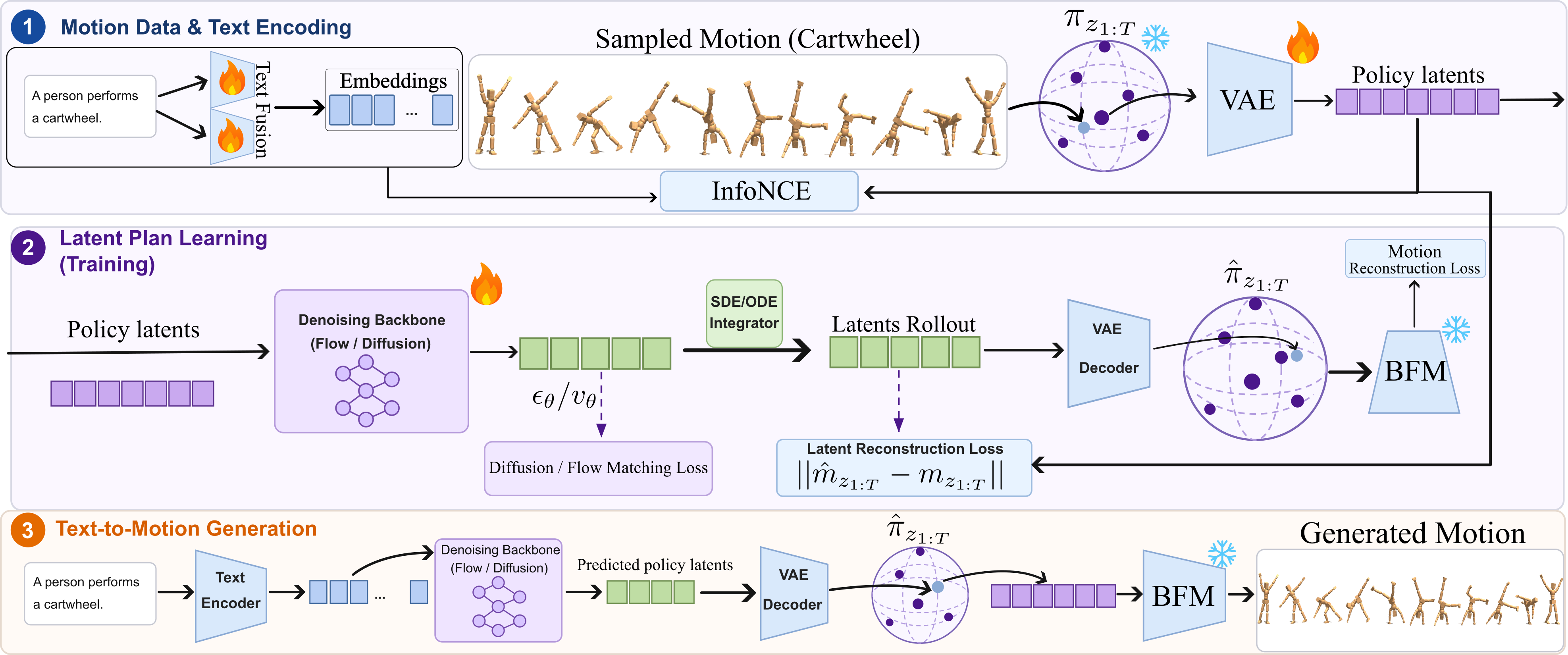}
    \caption{Text2BFM method and its principal diagram components. Shown are the training (steps 1 and 2) and the generation (3) pipelines.}
    \label{fig:diagram}
\end{figure*}

\section{Method}

Given a natural-language instruction $Y$, our goal is to generate a motion trajectory
$\tau=\{s_t\}_{t=1}^{T}$. Rather than synthesizing joint poses directly, we generate motion
through a frozen Behavioral Foundation Model (BFM), which provides a latent-conditioned policy
\begin{equation}
a_t \sim \pi_{\mathrm{BFM}}(a_t \mid s_t,z_t),
\qquad
s_{t+1}\sim p_{\mathrm{env}}(s_{t+1}\mid s_t,a_t),
\label{eq:bfm_policy}
\end{equation}
where $s_t$ is the humanoid state, $a_t$ is the control action, and $z_t$ is a local behavioral
latent. Since actions are applied between consecutive states, the policy-latent sequence has length
$T_z=T-1$. The generative model therefore produces executable behavioral commands
$z_{1:T_z}$, which are rolled out by the pretrained policy $\pi_{\mathrm{BFM}}$ and environment
dynamics $p_{\mathrm{env}}$.

For each training motion $\tau^i=\{s^i_t\}_{t=1}^{T}$, we infer a corresponding sequence of
BFM policy latents using the frozen MetaMotivo tracking procedure. Specifically, we apply the
BFM backward map $B$ to future states, average the resulting embeddings over a short lookahead
window, and project them to the normalized policy-latent space:
\begin{equation}
z^i_t =
\operatorname{Proj}_z
\left(
\frac{1}{H_t}
\sum_{k=0}^{H_t-1}
B(s^i_{t+1+k})
\right),
\qquad
H_t=\min(L,T-t),
\qquad
t=1,\ldots,T_z .
\label{eq:latent_extraction}
\end{equation}
Thus, $z^i_t$ summarizes the future-conditioned behavior to be executed from state $s^i_t$.
No latent is required after the final state.

Directly generating the full sequence $z_{1:T_z}$ remains a long-horizon sequence-modeling
problem. We therefore introduce a compact variational behavioral bottleneck that compresses
$z_{1:T_z}$ into a shorter motion program $m_{1:T_m}$, where $T_m < T_z$. The bottleneck is
trained to preserve the executable information needed to reconstruct BFM policy latents while also
aligning the compact program with the input text. This is suitable because BFM latents encode
behavioral intent rather than frame-level pose details; language-described motions typically consist
of a small number of semantic phases, making the latent trajectory redundant over time.

The variational encoder maps the policy-latent trajectory to a posterior distribution over compact
programs:
\begin{equation}
q_{\phi}(m_{1:T_m}\mid z_{1:T_z})
=
\mathcal{N}
\left(
\mu_{\phi}(z_{1:T_z}),
\operatorname{diag}\left(\sigma_{\phi}^{2}(z_{1:T_z})\right)
\right).
\label{eq:posterior}
\end{equation}
A compact program $m_{1:T_m}$ is sampled using the reparameterization trick and decoded back
into executable BFM latents:
\begin{equation}
\hat z_{1:T_z}=D_{\theta}(m_{1:T_m}).
\label{eq:decoder}
\end{equation}
The reconstruction loss, which incorporates a policy‑level reconstruction term, is defined as
\begin{equation}
\mathcal{L}_{\mathrm{rec}}
=
\frac{1}{T_z}
\sum_{t=1}^{T_z}
\|z_t-\hat z_t\|_2^2 +  \frac{\lambda_\pi}{T_z}
\sum_{t=1}^{T_z}  D_{\mathrm{KL}}
\left(\pi_{\mathrm{BFM}}(\cdot \mid s_t, \hat z_t)
\,\|\, \pi_{\mathrm{BFM}}(\cdot \mid s_t,z_t)
\right).
\label{eq:rec_loss}
\end{equation}
The compact program space is regularized with a standard Gaussian prior:
\begin{equation}
\mathcal{L}_{\mathrm{KL}}
=
D_{\mathrm{KL}}
\left(
q_{\phi}(m_{1:T_m}\mid z_{1:T_z})
\,\|\, \mathcal{N}(0,I)
\right).
\label{eq:kl_loss}
\end{equation}

To make the bottleneck language-discriminative, we use a motion-text contrastive objective with
frame-token matching. Let $m_i=\{m_{i,t}\}_{t=1}^{T_m}$ denote the compact program for sample
$i$, and let $Y_i=\{y_{i,k}\}_{k=1}^{K_i}$ denote its token-level text representation. Motion-program
tokens and text tokens are projected into a shared normalized space:
\begin{equation}
\tilde m_{i,t}
=
\frac{P_m(m_{i,t})}{\|P_m(m_{i,t})\|_2},
\qquad
\tilde y_{j,k}
=
\frac{P_y(y_{j,k})}{\|P_y(y_{j,k})\|_2}.
\label{eq:projection}
\end{equation}
For each motion-text pair $(i,j)$ in a batch, we compute a frame-level score by pooling over text
tokens:
\begin{equation}
F_{ijt}
=
\lambda_{\mathrm{tok}}
\log
\left(
\frac{1}{K_j}
\sum_{k=1}^{K_j}
\exp
\left(
\frac{\tilde m_{i,t}^{\top}\tilde y_{j,k}}{\lambda_{\mathrm{tok}}}
\right)
\right),
\label{eq:token_pooling}
\end{equation}
where $\lambda_{\mathrm{tok}}$ controls the sharpness of token pooling. The normalization by $K_j$
prevents longer text descriptions from receiving systematically larger scores.

We then compute frame-importance weights from the frame scores:
\begin{equation}
w_{ijt}
=
\frac{
\exp(F_{ijt}/\lambda_{\mathrm{frm}})
}{
\sum_{u=1}^{T_m}
\exp(F_{iju}/\lambda_{\mathrm{frm}})
},
\label{eq:frame_weights}
\end{equation}
where $\lambda_{\mathrm{frm}}$ controls how concentrated the frame weighting is. The final similarity
between motion program $i$ and text $j$ is
\begin{equation}
R_{ij}
=
\sum_{t=1}^{T_m}
w_{ijt}F_{ijt}.
\label{eq:similarity}
\end{equation}

Let $\gamma=\exp(\alpha)$ be a positive learnable logit scale. For a batch of $B$ paired
motion-text examples, the bidirectional contrastive loss is
\begin{equation}
\mathcal{L}_{m\rightarrow Y}
=
-\frac{1}{B}
\sum_{i=1}^{B}
\log
\frac{
\exp(\gamma R_{ii})
}{
\sum_{j=1}^{B}\exp(\gamma R_{ij})
},
\qquad
\mathcal{L}_{Y\rightarrow m}
=
-\frac{1}{B}
\sum_{i=1}^{B}
\log
\frac{
\exp(\gamma R_{ii})
}{
\sum_{j=1}^{B}\exp(\gamma R_{ji})
}.
\label{eq:contrastive}
\end{equation}
The semantic alignment loss is
\begin{equation}
\mathcal{L}_{\mathrm{sem}}
=
\frac{1}{2}
\left(
\mathcal{L}_{m\rightarrow Y}
+
\mathcal{L}_{Y\rightarrow m}
\right).
\label{eq:semantic_loss}
\end{equation}
This local matching objective encourages the compact program to align with the relevant text tokens
at different behavioral phases, rather than relying only on a single global text-motion similarity.

The full bottleneck objective is
\begin{equation}
\mathcal{L}_{\mathrm{VBB}}
=
\mathcal{L}_{\mathrm{rec}}
+
\beta\mathcal{L}_{\mathrm{KL}}
+
\lambda_{\mathrm{sem}}\mathcal{L}_{\mathrm{sem}}.
\label{eq:vbb_loss}
\end{equation}

After training the text-aligned behavioral bottleneck, we train a conditional generator in the compact
program space. For each training pair, we sample a target compact program
$m\sim q_{\phi}(m\mid z_{1:T_z})$ and a noise sample $\epsilon\sim\mathcal{N}(0,I)$ with the same
shape as $m$. We use flow matching with the linear interpolation path
\begin{equation}
m(r)=(1-r)\epsilon + r m,
\qquad r\sim\mathcal{U}(0,1).
\label{eq:linear_path}
\end{equation}
The text-conditioned vector field $v_{\eta}$ is trained to predict the constant flow from $\epsilon$
to $m$:
\begin{equation}
\mathcal{L}_{\mathrm{FM}}
=
\mathbb{E}_{\epsilon,m,r,Y}
\left[
\left\|
v_{\eta}(m(r),r,Y)
-
(m-\epsilon)
\right\|_2^2
\right].
\label{eq:flow_matching}
\end{equation}
Thus, the flow model learns to map Gaussian noise to a text-consistent compact behavioral program.
Training proceeds in three stages. First, we use a pretrained BFM and keep it frozen. Second, for
each paired text-motion example $(Y^i,\tau^i)$, we infer the policy-latent trajectory
$z^i_{1:T_z}$ and train the encoder, decoder, and semantic projection modules using
$\mathcal{L}_{\mathrm{VBB}}$. Third, we freeze the bottleneck decoder and train the conditional flow
model using $\mathcal{L}_{\mathrm{FM}}$.
At inference time, given a text instruction $Y$ and an initial humanoid state $s_1$ fixed by the
generation protocol, we sample $m(0)\sim\mathcal{N}(0,I)$ and solve the flow ODE
\begin{equation}
\frac{dm(r)}{dr}
=
v_{\eta}(m(r),r,Y),
\qquad r\in[0,1],
\label{eq:flow_ode}
\end{equation}
to obtain a generated compact program $m(1)$. The decoder maps this program to executable policy
latents $\hat z_{1:T_z}=D_{\theta}(m(1))$. The final motion is produced by rolling out the frozen
BFM policy using $\hat z_t$ at each timestep.

\section{Experiments}

\subsection{Experimental Setup}

We evaluate \textbf{Text2BFM} on two standard text-to-motion benchmarks: HumanML3D~\citep{guo2022humanml3d} and KIT-ML~\citep{plappert2016kit}. HumanML3D contains $14{,}616$ motion sequences with $44{,}970$ text descriptions, and KIT-ML contains $3{,}911$ motions with $6{,}278$ annotations. We follow the standard train/test splits and evaluation protocol used in prior work~\citep{zhang2022motiondiffuse, tevet2023mdm, chen2023mld, zhang2023t2mgpt, mo2023momask}. For all stochastic evaluations of our method, we report the mean and $95\%$ confidence interval over $20$ independent evaluation runs with different random seeds.
We report the standard metrics used in text-to-motion evaluation: \textbf{R-Precision} for text-motion retrieval accuracy, \textbf{FID} for distributional similarity to real motions, \textbf{MultiModal Distance} for text-motion feature alignment, and \textbf{MultiModality} for generation diversity under the same prompt.


The BFM policy is pretrained on the HY-Motion dataset~\citep{hy2024hymotion} following the MetaMotivo training procedure~\citep{metamotivo}, and is kept frozen during text-to-motion training.    On a single NVIDIA A100 80GB GPU, behavioral bottleneck training takes $17$ hours and text-to-plan generator training takes $28$ hours. BFM latent extraction time is reported separately once preprocessing is finalized. The full list of hyperparameters and provided in Table \ref{tab:core_arch}.

\begin{table*}[hbt!]
\centering
\caption{Quantitative comparison on HumanML3D and KIT-ML datasets. Metrics: R-Precision Top-3 and MultiModality (the higher, the better), FID and MultiModal Distance (the lower, the better). Best values are highlighted in \textcolor{LightBlue}{blue}.}
\label{tab:main_results}
\setlength{\tabcolsep}{3.3pt}
\begin{tabular}{llcccc}
\toprule
\textbf{Dataset} & \textbf{Method} & \textbf{R-Prec. Top-3} $\uparrow$ & \textbf{FID} $\downarrow$ & \textbf{MM-Dist} $\downarrow$ & \textbf{MultiModality} $\uparrow$ \\
\midrule
\multirow{9}{*}{HumanML3D}
& TM2T\citep{guo2022tm2t} & $0.729\pmt{.002}$ & $1.501\pmt{.017}$ & $3.467\pmt{.011}$ & $2.424\pmt{.093}$ \\
& T2M\citep{guo2022humanml3d} & $0.736\pmt{.002}$ & $1.087\pmt{.021}$ & $3.347\pmt{.008}$ & $2.219\pmt{.074}$ \\
& MDM\citep{tevet2023mdm} & $0.611\pmt{.007}$ & $0.544\pmt{.044}$ & $5.566\pmt{.027}$ & \textbf{\textcolor{LightBlue}{$2.799\pmt{.072}$}} \\
& MLD\citep{chen2023mld} & $0.772\pmt{.002}$ & $0.473\pmt{.013}$ & $3.196\pmt{.010}$ & $2.413\pmt{.079}$ \\
& MotionDiffuse\citep{zhang2022motiondiffuse} & $0.782\pmt{.001}$ & $0.630\pmt{.001}$ & $3.113\pmt{.001}$ & $1.553\pmt{.042}$ \\
& T2M-GPT\citep{zhang2023t2mgpt} & $0.775\pmt{.002}$ & $0.141\pmt{.005}$ & $3.121\pmt{.009}$ & $1.831\pmt{.048}$ \\
& ReMoDiffuse\citep{zhang2023remodiffuse} & $0.795\pmt{.004}$ & $0.103\pmt{.004}$ & $2.974\pmt{.016}$ & $1.795\pmt{.043}$ \\
& MoMask\citep{guo2024momask} & $0.807\pmt{.002}$ & \textbf{\textcolor{LightBlue}{$0.045\pmt{.002}$}} & $2.958\pmt{.008}$ & $1.241\pmt{.040}$ \\
& \textbf{Text2BFM} 
& \textbf{\textcolor{LightBlue}{$0.876\pmt{.005}$}} 
& $1.172\pmt{.013}$ 
& \textbf{\textcolor{LightBlue}{$2.498\pmt{.061}$}} 
& $1.233\pmt{.085}$ \\
\midrule
\multirow{9}{*}{KIT-ML}
& TM2T\citep{guo2022tm2t} & $0.587\pmt{.005}$ & $3.599\pmt{.153}$ & $4.591\pmt{.026}$ & \textbf{\textcolor{LightBlue}{$3.292\pmt{.081}$}} \\
& T2M\citep{guo2022humanml3d} & $0.681\pmt{.007}$ & $3.022\pmt{.107}$ & $3.488\pmt{.028}$ & $2.052\pmt{.107}$ \\
& MDM\citep{tevet2023mdm} & $0.396\pmt{.004}$ & $0.497\pmt{.021}$ & $9.191\pmt{.022}$ & $1.907\pmt{.214}$ \\
& MLD\citep{chen2023mld} & $0.734\pmt{.007}$ & $0.404\pmt{.027}$ & $3.204\pmt{.027}$ & $2.192\pmt{.071}$ \\
& MotionDiffuse\citep{zhang2022motiondiffuse} & $0.739\pmt{.004}$ & $1.954\pmt{.062}$ & $2.958\pmt{.005}$ & $0.730\pmt{.013}$ \\
& T2M-GPT \citep{zhang2023t2mgpt} & $0.745\pmt{.006}$ & $0.514\pmt{.029}$ & $3.007\pmt{.023}$ & $1.570\pmt{.039}$ \\
& ReMoDiffuse \citep{zhang2023remodiffuse} & $0.765\pmt{.055}$ & \textbf{\textcolor{LightBlue}{$0.155\pmt{.006}$}} & $2.814\pmt{.012}$ & $1.239\pmt{.028}$ \\
& MoMask \citep{guo2024momask} & $0.781\pmt{.005}$ & $0.204\pmt{.011}$ & $2.779\pmt{.022}$ & $1.131\pmt{.043}$ \\
& \textbf{Text2BFM} 
& \textbf{\textcolor{LightBlue}{$0.901\pmt{.008}$}} 
& $1.482\pmt{.098}$ 
& \textbf{\textcolor{LightBlue}{$2.658\pmt{.074}$}} 
& $1.367\pmt{.103}$ \\
\bottomrule
\end{tabular}
\end{table*}

\subsection{Main Results}

Table~\ref{tab:main_results} compares Text2BFM with representative text-to-motion methods. Prior approaches typically synthesize motion directly in pose space, kinematic latent spaces, or discrete motion-token spaces. In contrast, Text2BFM first generates compact behavioral programs, then decodes them into BFM policy latents, and finally realizes the motion through policy rollout.

This design decouples high-level semantic planning from low-level motion execution. Generating in the compact program space alleviates the burden of modeling long pose sequences, while the frozen BFM policy provides an executable prior that ensures temporal coherence and supports contact-rich motion. As a result, Text2BFM achieves the best R-Precision and MultiModal Distance scores on both datasets. However, the BFM policy and its environment-specific action space introduce a domain bias; this leads to FID values that are not optimal, as the distribution of the generated motions can deviate from the real motion data distribution.

\subsection{Compositional Motion Evaluation}

Long prompts often describe sequences of behaviors rather than a single action, e.g., ``walk forward, then turn left, then sit down.'' To evaluate this setting, we construct compositional prompts from HumanML3D by combining atomic text descriptions into multi-stage prompts. Given $N$ atomic instructions $\{Y^{(n)}\}_{n=1}^{N}$, we form
\begin{equation}
Y_{\mathrm{comp}}
=
Y^{(1)} \ \text{then}\ Y^{(2)} \ \text{then}\ \cdots \ \text{then}\ Y^{(N)} .
\end{equation}
We evaluate $N\in\{3,4\}$ stages and select semantically plausible action sequences, such as walk-turn-sit or reach-pick-place. 
 \textbf{Text2BFM} generates a single compact program from the full prompt $Y_{\mathrm{comp}}$. \textbf{Text2BFM-Compose} decomposes the prompt into clauses and generates one compact program per clause. For each clause $Y^{(n)}$, we sample
$
\hat m^{(n)} \sim p_{\eta}(m \mid Y^{(n)}),
\,
\hat z^{(n)}_{1:T_n}
=
D_{\theta}(\hat m^{(n)}).
$
The decoded policy-latent sequences are then concatenated:
\begin{equation}
\hat z^{\mathrm{comp}}_{1:T}
=
\hat z^{(1)}_{1:T_1}
\oplus
\hat z^{(2)}_{1:T_2}
\oplus
\cdots
\oplus
\hat z^{(N)}_{1:T_N}.
\end{equation}
To reduce discontinuities at clause boundaries, we blend a short overlap of $O$ latent steps. For the boundary between stages $n$ and $n+1$, the blended overlap is
\begin{equation}
\tilde z_o
=
(1-\rho_o)\hat z^{(n)}_{T_n-O+o}
+
\rho_o \hat z^{(n+1)}_{o},
\qquad
\rho_o=\frac{o}{O+1},
\qquad
o=1,\ldots,O .
\end{equation}
The boundary sequence is formed by replacing the overlapping latents with
$\{\tilde z_o\}_{o=1}^{O}$ before rollout. The final composed latent sequence is executed by the frozen BFM policy. This variant tests whether complex motions can be assembled from local behavior representations rather than generated as a single long pose sequence. To support our findings, we compare proposed approaches with a recent framework Kimodo \citep{Kimodo2026}. Figure \ref{fig:composition_cmp} showcases that Kimodo is unable to follow text specification precisely. Text2BFM-Compose covers more prompt stages and produces smoother boundaries in this example, although long $N=4$ compositions remain challenging quantitatively.

\begin{figure}[hbt]
    \centering
    \begin{tabular}{@{} c r @{}}
        \adjustbox{valign=c}{Kimodo}
            & \includegraphics[width=0.7\linewidth, valign=c]{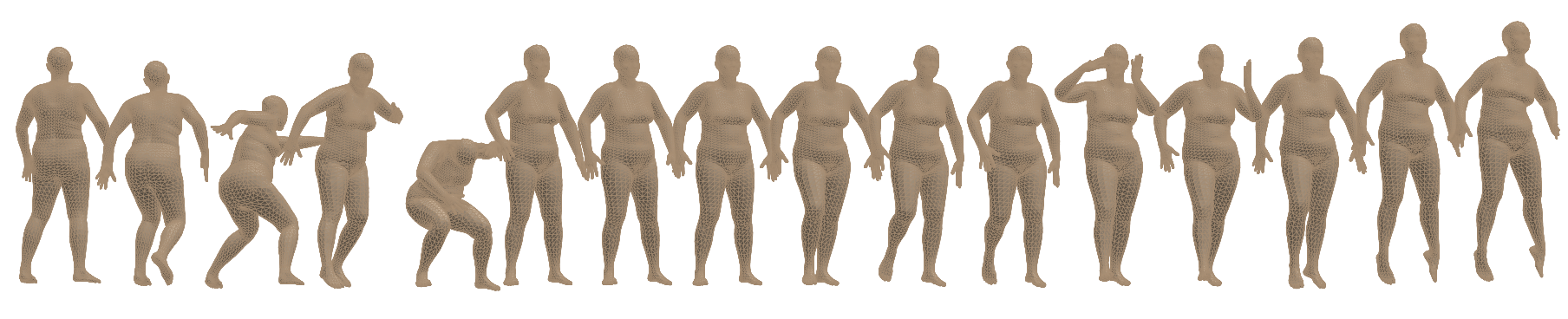} \\[4pt]
        \adjustbox{valign=c}{Text2BFM}
            & \includegraphics[width=0.7\linewidth, valign=c]{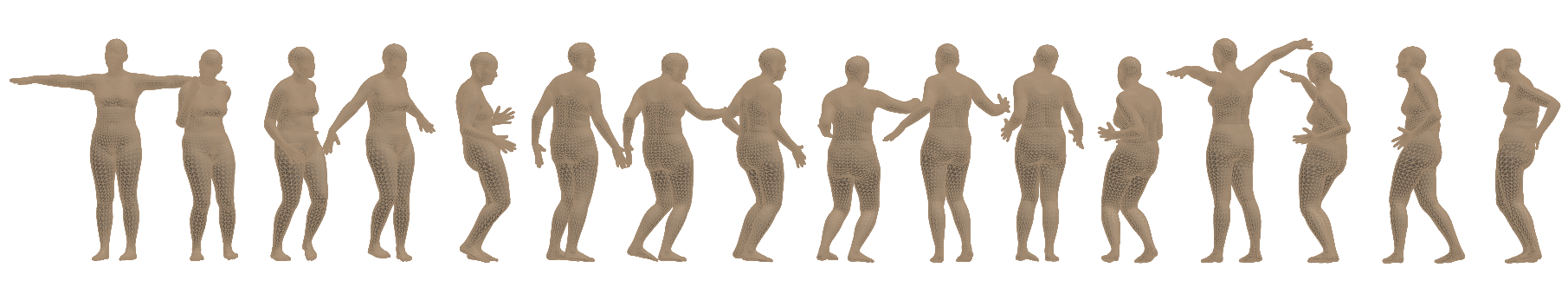} \\[4pt]
        \adjustbox{valign=c}{Text2BFM-Compose}
            & \includegraphics[width=0.7\linewidth, valign=c]{figures/composition_mesh.png}   \\
    \end{tabular}
    \caption{Comparison of methods for compositional motion generation. The text prompt is ``A person runs forward, then turns left, does a quick, sharp punch, and 
follows it with a kick to the right. After that, he walks forward, 
raises his arms in victory, and finishes by waving happily''.}
    \label{fig:composition_cmp}
\end{figure}

To quantitatively measure and compare performances of different methods, we perform the following procedure. Generated motions are divided into $N$ segments, using known composition boundaries for Text2BFM-Compose and uniform segments otherwise. \textbf{Order Accuracy} measures whether the best-matching clause for each generated segment follows the correct temporal order. Order Accuracy is computed by embedding each generated segment and each clause using a frozen embedder. For each segment, we assign the clause with the highest similarity. The order is correct if the assigned clause indices are strictly increasing and match the target sequence. \textbf{Transition Score} measures the smoothness around the action boundaries as the average velocity discontinuity around each clause boundary:
\begin{equation}
\mathrm{Transition}
=
\frac{1}{N-1}
\sum_{i=1}^{N-1}
\left(
\|s_{i+1}-s_i\|_2
+ \|v_{i+1}-v_i\|_2
\right).
\end{equation}

\begin{table}[t]
\centering
\caption{Evaluation on longer compositional prompts with $N\geq 3$ stages. Longer prompts test whether the model preserves semantic ordering over extended horizons.}
\label{tab:composition_long}
\begin{tabular}{lcccc}
\toprule
\textbf{Method} & \textbf{$N$} & \textbf{Order Acc.} $\uparrow$ & \textbf{MM-Dist} $\downarrow$ & \textbf{Transition} $\downarrow$ \\
\midrule
Text2BFM & $3$ & 0.395 &  4.476 & 0.222 \\
\textbf{Text2BFM-Compose} & $3$ & \textcolor{LightBlue}{0.671} & \textcolor{LightBlue}{3.421} & \textcolor{LightBlue}{0.047} \\
\midrule
Text2BFM & $4$ & 0.322 &  6.579 & 0.237 \\
\textbf{Text2BFM-Compose} & $4$ & \textcolor{LightBlue}{0.509} &  \textcolor{LightBlue}{4.61} & \textcolor{LightBlue}{0.037} \\
\bottomrule
\end{tabular}
\end{table}
The compositional benchmark highlights a difference between generating poses and generating executable behavior programs. Pose-space generators must model the full long-horizon trajectory and all transitions directly from a single text embedding. In contrast, Text2BFM represents each phase through local executable latents, making it easier to preserve the identity and order of individual actions. The explicit composition variant further benefits from the modularity of the compact program space: each clause is first mapped to a behavioral program, and the resulting policy-latent sequences are combined before rollout.


\subsection{Ablation Study}

We ablate the main components of Text2BFM: compression of BFM policy-latent sequences, semantic alignment, and the compact-program generator. All variants use the same frozen BFM policy and differ only in the text-to-program pipeline.

\paragraph{Latent sequence compression.}
We first study how the BFM latent trajectory $z_{1:T}$ is compressed into a compact program. We compare four variants: no compression, average pooling, a variational behavioral bottleneck without semantic alignment, and the full text-aligned bottleneck. The average pooling operation directly follows from the conventional BFM objective,
$
m_{\mathrm{avg}}
=
1/T
\sum_{t=1}^{T} z_t ,
$
which removes temporal ordering.  The full model uses the proposed variational bottleneck with KL regularization and motion-text semantic alignment. Table~\ref{tab:ablation_compression} shows that preserving temporal structure is important: average pooling performs poorly because it collapses the latent trajectory into an unordered summary. Generating a full latents sequence without compression improves the motion quality, but does not organize the latent space according to language. Adding the variational bottleneck improves performance, and semantic alignment further improves text-motion matching, as reflected by lower MM-Dist and higher R-Precision.

\begin{table}[t]
\centering
\caption{Ablation on policy-latent sequence compression. All variants use the same frozen BFM policy and the same evaluation protocol. Best values are highlighted in \textcolor{LightBlue}{blue}.}
\label{tab:ablation_compression}
\begin{tabular}{lccc}
\toprule
\textbf{Compression} & \textbf{FID} $\downarrow$ & \textbf{MM-Dist} $\downarrow$ & \textbf{R-Prec.} $\uparrow$ \\
\midrule
Average pooling 
& $2.373\,{\color{gray}\scriptstyle \pm .010}$ 
& $3.573\,{\color{gray}\scriptstyle \pm .006}$ 
& $0.786\,{\color{gray}\scriptstyle \pm .070}$ \\

No Compression 
& $1.689\,{\color{gray}\scriptstyle \pm .007}$ 
& $2.869\,{\color{gray}\scriptstyle \pm .006}$ 
& $0.743\,{\color{gray}\scriptstyle \pm .004}$ \\

VBB w/o semantic loss 
& $1.272\,{\color{gray}\scriptstyle \pm .007}$ 
& $2.832\,{\color{gray}\scriptstyle \pm .025}$ 
& $0.863\,{\color{gray}\scriptstyle \pm .003}$ \\

\textbf{Text-aligned VBB} 
& ${\textcolor{LightBlue}{1.172}}\,{\color{gray}\scriptstyle \pm .013}$ 
& ${\textcolor{LightBlue}{2.498}}\,{\color{gray}\scriptstyle \pm .061}$ 
& ${\textcolor{LightBlue}{0.877}}\,{\color{gray}\scriptstyle \pm .005}$ \\
\bottomrule
\end{tabular}
\end{table}

\paragraph{Effect of semantic alignment.}
The comparison between VAE compresson w/o semantic loss ($\mathcal{L}_\text{sem}$) and the full model isolates the role of the motion-text contrastive objective (Table \ref{tab:ablation_compression}). Without this loss, the compact program is optimized mainly to reconstruct BFM latents, but its geometry is not explicitly aligned with text. The full model improves semantic retrieval and text-motion distance, suggesting that language alignment makes behavior representations $\pi_z$ more suitable for text-conditioned generation.

\paragraph{Compression factors.}
We evaluate how the compression factor of the behavioral encoder affects reconstruction of the BFM policy-latent trajectory. Given $z_{1:T}$, the encoder produces a shorter compact program, and the decoder reconstructs $\hat z_{1:T}$. We compare compression factors $\{4,8,16\}$ using two metrics: latent reconstruction error and policy-level Action KL. The latter measures whether the reconstructed latents induce similar action distributions under the frozen BFM policy:
\begin{equation}
\mathrm{ActionKL}
=
\frac{1}{T}
\sum_{t=1}^{T}
D_{\mathrm{KL}}
\left(
\pi_{\mathrm{BFM}}(\cdot \mid s_t, \hat z_t)
\,\|\, 
\pi_{\mathrm{BFM}}(\cdot \mid s_t, z_t)
\right).
\end{equation}
\begin{table}[t]
\centering
\caption{Effect of temporal compression on BFM latent reconstruction and policy-level behavior preservation. Lower is better.}
\label{tab:compression_ablation}
\begin{tabular}{ccc}
\toprule
\textbf{Compression} & \textbf{Reconstruction} $\downarrow$ & \textbf{Action KL} $\downarrow$ \\
\midrule
$4\times$  & 0.187 & 1.04 \\
$8\times$  & 0.194 & 1.05 \\
$16\times$ & 0.342 & 2.02 \\
\bottomrule
\end{tabular}
\end{table}
\noindent Table~\ref{tab:compression_ablation} shows a trade-off between compactness and fidelity. The $4\times$ and $8\times$ models preserve both latent reconstruction and induced policy behavior, while $16\times$ compression substantially increases Action KL. We therefore use the compression factor that best balances compact behaviors $\pi_z$ length and policy-level preservation.

\paragraph{Flow Matching VS Diffusion.}
Table \ref{tab:ablation_generator} provides comparison of the main generation backbone. Flow matching improves FID and MultiModal Distance and requires fewer sampling steps, while diffusion obtains slightly higher R-Precision.

\begin{table}[ht]
\centering
\caption{Ablation of the choice of the underlying generator. Both variants operate in the learned behavioral bottleneck space.}
\label{tab:ablation_generator}
\begin{tabular}{lcccc}
\toprule
\textbf{Generator} & \textbf{FID} $\downarrow$ & \textbf{MM-Dist} $\downarrow$ & \textbf{R-Prec.} $\uparrow$ & \textbf{Sampling Steps} $\downarrow$ \\
\midrule
Diffusion in $\pi_z$-space 
& $1.672\pmt{.011}$ 
& $2.617\pmt{.064}$ 
& \textbf{\textcolor{LightBlue}{$0.885\pmt{.003}$}} 
& $50$ \\
\textbf{Flow matching in $\pi_z$-space} 
& \textbf{\textcolor{LightBlue}{$1.172\pmt{.013}$}} 
& \textbf{\textcolor{LightBlue}{$2.498\pmt{.061}$}} 
& $0.8762\pmt{.005}$ 
& \textbf{\textcolor{LightBlue}{16}} \\
\bottomrule
\end{tabular}
\end{table}





\subsection{Conclusion \& Limitations}
In the current work we presented Text2BFM - a text to motion generation model that aligns pretrained BFM policy representation space together with textual descriptions. This enables motion generation for long, composite textual descriptions. 
However, Text2BFM inherits both the strengths and limitations of the pretrained BFM (morphology, simulator dynamics, action space, and training-data bias). If requested behavior is outside the BFM latent space, the text-conditioned generator cannot reliably produce it. Consequently, comparisons to methods trained only on HumanML3D/KIT-ML are not equal-data comparisons. Finally, rare motions, acrobatic actions, and object-dependent interactions remain challenging when they are underrepresented in the motion data or not well captured by the BFM policy. Future work will explore joint refinement of the BFM and text-conditioned generator, explicit object-interaction modeling, and extensions to multi-character motion generation.


\paragraph{Code and reproducibility.}
We provide our code in the supplementary material. It contains the implementations used for the main experiments reported in the paper and is sufficient to reproduce the presented results.

\paragraph{Existing assets.}
We use HumanML3D\citep{guo2022humanml3d}, KIT-ML\citep{plappert2016kit},AMASS\citep{AMASS:ICCV:2019},MotionHub\citep{ling2024motionllama} and the public Metamotivo \citep{metamotivo} model according to their official terms of use and cite their original creators. We do not redistribute third-party datasets, baseline code, or checkpoints; baseline numbers are taken from the corresponding papers. Third-party software dependencies are listed in the supplementary code README. We refer to Appendix \ref{fig:composition_dataset} on how exactly dataset used for training was processed.

\paragraph{Broader impacts.}
Text2BFM may have positive societal impacts by making high-quality character animation more accessible for applications such as virtual avatars, games, education, assistive interfaces, and simulation-based robotics research. By generating motion through executable behavioral policies, it may also support safer prototyping of humanoid behaviors in simulation before physical deployment. Potential negative impacts include deceptive avatar animation and unsafe transfer of simulated behaviors to real robots without validation.


\clearpage
\bibliographystyle{plainnat}
\bibliography{references}

\newpage
\appendix

\section{Technical Details and Hyperparameters}
\label{sec:tech_details}

Our method is trained in two stages: 1) semantic latent motion VAE pretraining, and 2) text-to-latent generation ($text \rightarrow m \rightarrow z$).

\paragraph{Stage 1: Semantic VAE.}
The backbone is a causal 1D convolutional encoder-decoder with residual temporal blocks and hierarchical downsampling. The temporal axis is compressed by a factor of $8$. The latent sequence $m$ has stochastic posterior parameterization ($\mu,\log\sigma^2$) and reparameterized sampling.

\paragraph{Stage 2: Text-to-latent generator.}
A pretrained FB backbone is adapted to predict latent motion sequences conditioned on text. Optimization uses composite generation/reconstruction losses (including consistency in both $m$ and $z$ spaces).

\paragraph{Compute resources.}
All experiments were run on a single NVIDIA A100 80GB GPU with 32 CPU cores and 256GB system RAM. Offline BFM policy-latent extraction took approximately 3 hours for HumanML3D and 1 hour for KIT-ML. Training the behavioral bottleneck and text-to-plan flow generator took 17 and 28 hours, respectively, and evaluation took approximately 2--3 hours across HumanML3D and KIT-ML, including repeated sampling for confidence intervals. A full reproduction of the main Text2BFM results requires approximately 50 A100 GPU-hours, excluding dataset download and environment setup.

\begin{table*}[hbt]
\centering
\caption{Core architecture and hyperparameters.}
\label{tab:core_arch}
\begin{tabular}{lll}
\hline
Category & Hyperparameter & Value \\
\hline
VAE & Batch size & 256 \\
VAE & Input width & 256 \\
VAE & Latent dim $d_m$ & 48 \\
VAE & Latent ratio / min / max & 0.25 / 32 / 48 \\
VAE & Stride per level & 2 \\
VAE & Channel width / residual depth & 256 / 3 \\
VAE & Dilation growth rate & 3 \\
VAE & Activation / norm & ReLU / GroupNorm \\
VAE & Residual dropout & 0.1 \\
VAE & Padding mode & replicate \\
VAE & Learning rate / weight decay & $5\times10^{-5}$ / $5\times10^{-4}$ \\
VAE & Loss coef. $\lambda_\pi$ & 0.1 \\
VAE & Loss coef. $\lambda_\text{sem}$ & 0.35 \\
VAE & KL regularization $\beta$ & $10^{-4}$ \\
\hline
Generator  & Batch size & 256 \\
Generator  & Learning rate / weight decay & $3\times10^{-5}$ / $5\times10^{-4}$ \\
Generator  & LR warmup / start factor & 2000 / 0.1 \\
Generator  & Min LR ratio / LR decay steps & 0.1 / 15000 \\
Generator  & Early stopping patience / min delta & 12 / $10^{-4}$ \\
Generator  & Feature width / depth / heads & 512 / 8 / 8 \\
Generator  & Dropout / MLP ratio & 0.1 / 4.0 \\
Generator  & MLP activation & \texttt{gelu\_tanh} \\
Generator  & QK normalization / QKV bias & \texttt{rms} / True \\
Generator  & Start token / long skip connection & False / False \\
\hline
Conditioning & Text token dim / context dim & 768 / 1024 \\
Conditioning & Text adapter (width / depth / heads) &  768 / 2 / 8 \\
Conditioning & Text adapter dropout & 0.1 \\
Conditioning & Temporal mask type / span & \texttt{narrowband} / 2.0 \\
Conditioning & Condition dropout probability & 0.2 \\
Conditioning & RoPE on single branch / time factor & True / 1.0 \\
Conditioning & Null context length & 1 \\
\hline
Sampling & Guidance scale & 1.5 \\
Sampling & Sampling steps / solver & 16 / \texttt{euler}
\end{tabular}
\end{table*}

\section{Theoretical Motivation for Text-Aligned Behavioral Compression}
\label{sec:theory}

This section motivates the compression of BFM policy-latent trajectories into compact text-aligned behavioral programs. The key point is that MetaMotivo-style latents are not arbitrary framewise codes. They are future-conditioned behavioral contexts obtained from a frozen backward representation and then projected to a normalized policy-latent space. Thus, a latent $z_t$ should be interpreted as a local command describing what behavior should unfold from state $s_t$, rather than as a direct encoding of the current pose.

For notational simplicity, we denote the policy-latent sequence length by $T$. In the main text, this corresponds to $T_z = T_{\mathrm{motion}} - 1$. Given a motion trajectory $s_{0:T}$, the tracked BFM pseudo-labels are constructed as
$$
\bar z_t
=
\frac{1}{H_t}
\sum_{k=0}^{H_t-1}
B_\psi(s_{t+1+k}),
\qquad
H_t=\min(L,T-t),
$$
followed by projection to the normalized latent sphere:
$$
z_t=\operatorname{Proj}_z(\bar z_t),
\qquad
\operatorname{Proj}_z(u)
=
\sqrt{d_z}\frac{u}{\|u\|_2}.
$$
Hence, $z_{1:T}$ is a temporally ordered curve on the sphere of radius $\sqrt{d_z}$. Its complexity is better described by its path length than by the ambient latent dimension. We measure this path length using total variation:
$$
\operatorname{TV}(z_{1:T})
=
\sum_{t=1}^{T-1}
\|z_{t+1}-z_t\|_2.
$$
Low total variation means that behavioral intent changes gradually. Such trajectories are naturally compressible: neighboring latents often communicate nearly identical control information to the frozen BFM policy.

The future-conditioned construction provides an explicit smoothing mechanism. When $H_t=L$, adjacent averages satisfy a telescoping identity:
$$
\bar z_{t+1}-\bar z_t
=
\frac{1}{L}
\left(
B_\psi(s_{t+L+1})-B_\psi(s_{t+1})
\right).
$$
Thus, most intermediate future terms cancel. If $B_\psi$ is $L_B$-Lipschitz and $\|\bar z_t\|_2\geq \rho>0$, then the projection map is $\frac{2\sqrt{d_z}}{\rho}$-Lipschitz on this region, yielding
$$
\|z_{t+1}-z_t\|_2
\leq
\frac{2\sqrt{d_z}L_B}{\rho L}
\|s_{t+L+1}-s_{t+1}\|_2.
$$
Consequently, if the motion has bounded future-window displacement,
$$
\sum_{t=1}^{T-1}
\|s_{t+L+1}-s_{t+1}\|_2
\leq
C_{\mathrm{motion}},
$$
then
$$
\operatorname{TV}(z_{1:T})
\leq
\frac{2\sqrt{d_z}L_B}{\rho L}
C_{\mathrm{motion}}.
$$
This shows that BFM latents inherit a smoothing bias from future averaging: rapid pose-level fluctuations are partially averaged out, while persistent changes in behavioral intent remain.

We now connect latent compression to rollout quality. Let the closed-loop BFM dynamics be
$$
s_{t+1}=F(s_t,z_t),
\qquad
F(s,z)=f(s,\pi_{\mathrm{BFM}}(s,z)),
$$
where $f$ is the environment dynamics and $\pi_{\mathrm{BFM}}$ is the frozen BFM policy. Assume $F$ is Lipschitz:
$$
\|F(s,z)-F(s',z')\|_2
\leq
L_s\|s-s'\|_2
+
L_z\|z-z'\|_2.
$$
Here, $L_s$ measures closed-loop sensitivity to state perturbations, while $L_z$ measures sensitivity to latent reconstruction error.

\begin{proposition}[Compact latent plans yield bounded rollout error]
\label{prop:compression_rollout_short}
Let $z_{1:T}$ be a BFM policy-latent trajectory with total variation
$$
V=\operatorname{TV}(z_{1:T}).
$$
For any $m\geq 2$, there exists a piecewise-constant approximation $\tilde z_{1:T}$ with at most $m$ contiguous segments such that
$$
\max_t \|z_t-\tilde z_t\|_2
\leq
\frac{V}{m-1}.
$$
Let $s_{1:T}$ and $\tilde s_{1:T}$ be rollouts from the same initial state using $z_{1:T}$ and $\tilde z_{1:T}$, respectively. Then
$$
\|\tilde s_t-s_t\|_2
\leq
L_z
\frac{V}{m-1}
\sum_{j=1}^{t-1}
L_s^{t-1-j}.
$$
In particular, if $L_s<1$, then
$$
\|\tilde s_t-s_t\|_2
\leq
\frac{L_z V}{(m-1)(1-L_s)}.
$$
\end{proposition}

\begin{proof}
Let
$$
\delta=\frac{V}{m-1}.
$$
Partition the latent sequence into contiguous segments such that the total variation inside each segment is at most $\delta$. Define $\tilde z_t$ as the first latent vector in the segment containing $t$. If the segment begins at index $a$, then
$$
\|z_t-\tilde z_t\|_2
=
\|z_t-z_a\|_2
\leq
\sum_{j=a}^{t-1}
\|z_{j+1}-z_j\|_2
\leq
\delta.
$$
Thus,
$$
\max_t\|z_t-\tilde z_t\|_2
\leq
\frac{V}{m-1}.
$$

Let
$$
e_t=\|\tilde s_t-s_t\|_2.
$$
Since both rollouts begin from the same initial state, $e_1=0$. By Lipschitz continuity,
$$
e_{t+1}
=
\|F(\tilde s_t,\tilde z_t)-F(s_t,z_t)\|_2
\leq
L_s e_t
+
L_z\|\tilde z_t-z_t\|_2.
$$
Unrolling the recursion gives
$$
e_t
\leq
L_z
\sum_{j=1}^{t-1}
L_s^{t-1-j}
\|\tilde z_j-z_j\|_2.
$$
Using the uniform approximation bound yields
$$
e_t
\leq
L_z
\frac{V}{m-1}
\sum_{j=1}^{t-1}
L_s^{t-1-j}.
$$
If $L_s<1$, the geometric sum is bounded by $\frac{1}{1-L_s}$, giving
$$
e_t
\leq
\frac{L_z V}{(m-1)(1-L_s)}.
$$
\end{proof}

The proposition gives a conservative explanation for why short latent programs can remain executable. The relevant quantity is not the dimensionality of the latent vector, but the temporal variation of the latent curve. If behavioral intent changes slowly and the closed-loop system is stable, then a compressed latent program induces only controlled rollout error. A learned decoder is more expressive than the piecewise-constant construction in the proof: it can generate smooth curves, reuse motifs, model phase structure, and allocate capacity to transitions such as contacts or direction changes.

The variational behavioral bottleneck can be interpreted as a rate--distortion relaxation. Let $Z=z_{1:T}$ be the tracked latent trajectory and let $M$ be the compact behavioral program. An encoder $q_\phi(M|Z)$ and decoder $p_\theta(Z|M)$ optimize a tradeoff of the form
$$
\mathcal{L}_{\mathrm{VBB}}
=
\mathbb{E}_{q_\phi(M|Z)}
\left[
D(Z,\hat Z)
\right]
+
\beta
D_{\mathrm{KL}}
\left(
q_\phi(M|Z)
\|p(M)
\right).
$$
The reconstruction term preserves the information needed to recover executable BFM latents, while the KL term prevents $M$ from becoming a framewise copy of $Z$. The bottleneck therefore encourages the model to store reusable behavioral factors such as speed, direction, gait, contact structure, and style.

However, reconstruction alone does not guarantee that the compressed program preserves the aspects of motion that matter for language. A purely reconstructive code may spend capacity on details that improve latent reconstruction but are weakly discriminative under text. Semantic alignment biases the bottleneck toward language-relevant behavioral factors.

\begin{proposition}[Semantic alignment preserves text discrimination under a margin]
\label{prop:semantic_margin}
Let $e_Y$ be a unit-normalized embedding of the correct text prompt and let $e_M$ be a unit-normalized embedding of the corresponding compact behavioral program. Suppose
$$
1-e_Y^\top e_M
\leq
\eta.
$$
Let $e_Y^-$ be the unit-normalized embedding of an incorrect prompt. If
$$
e_Y^\top e_Y^-
\leq
1-\Delta
$$
and
$$
\Delta>\eta+\sqrt{2\eta},
$$
then
$$
e_M^\top e_Y
>
e_M^\top e_Y^-.
$$
\end{proposition}

\begin{proof}
Since $e_Y$ and $e_M$ are unit-normalized,
$$
\|e_M-e_Y\|_2^2
=
2-2e_M^\top e_Y.
$$
The alignment assumption gives
$$
\|e_M-e_Y\|_2
\leq
\sqrt{2\eta}.
$$
For the negative prompt,
$$
e_M^\top e_Y^-
=
e_Y^\top e_Y^-
+
(e_M-e_Y)^\top e_Y^-.
$$
By Cauchy--Schwarz,
$$
(e_M-e_Y)^\top e_Y^-
\leq
\sqrt{2\eta}.
$$
Therefore,
$$
e_M^\top e_Y^-
\leq
1-\Delta+\sqrt{2\eta}.
$$
Meanwhile,
$$
e_M^\top e_Y
\geq
1-\eta.
$$
Hence,
$$
e_M^\top e_Y-e_M^\top e_Y^-
\geq
\Delta-\eta-\sqrt{2\eta}.
$$
If $\Delta>\eta+\sqrt{2\eta}$, the right-hand side is positive, proving the claim.
\end{proof}

For multiple negatives $\{e_{Y_i^-}\}_{i=1}^N$ satisfying
$$
e_Y^\top e_{Y_i^-}
\leq
1-\Delta,
$$
the same argument gives
$$
e_M^\top e_Y-e_M^\top e_{Y_i^-}
\geq
\Delta-\eta-\sqrt{2\eta}.
$$
With contrastive logits scaled by temperature $\tau$, this implies the lower bound
$$
p(Y|M)
\geq
\frac{1}
{
1+
N\exp
\left(
-\frac{\Delta-\eta-\sqrt{2\eta}}{\tau}
\right)
}.
$$
Thus, reducing alignment error increases the contrastive separability of the compact behavioral program.

Together, these arguments justify text-aligned behavioral compression. Future-conditioned BFM latents form smooth behavioral curves; bounded-variation curves admit compact approximations; stable closed-loop execution converts small latent error into controlled rollout error; and semantic alignment ensures that the compressed program preserves the factors of motion that are discriminative under language. The resulting representation is therefore a compact semantic control program: low-rate enough to compress motion, executable enough to drive the frozen BFM policy, and aligned enough to support text-conditioned generation.

\section{Dataset Structure}
\begin{figure}[hbt]
    \centering
    \includegraphics[scale=0.4]{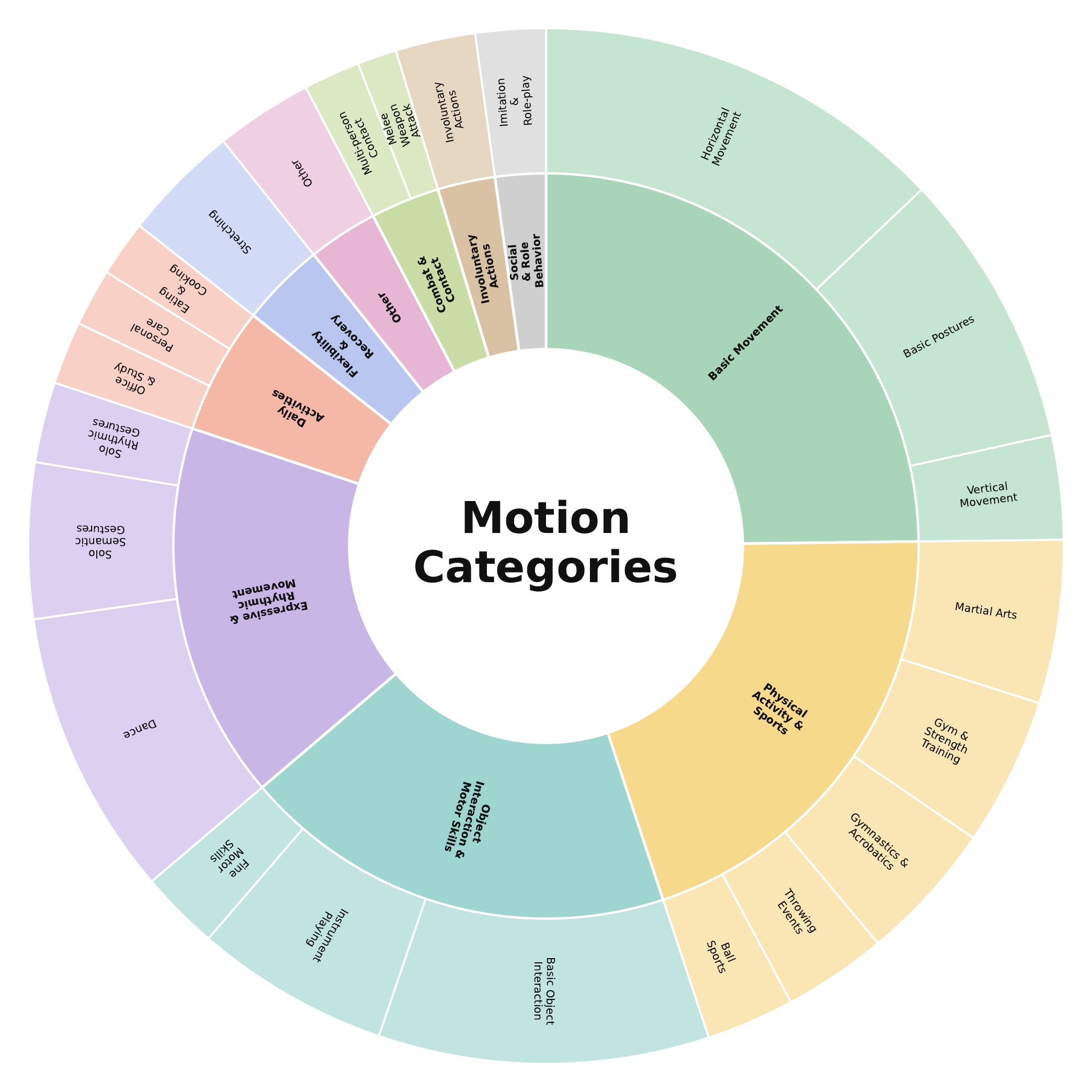} 
    \caption{Hierarchical visualization of motion categories in the dataset. The dataset was constructed from the AMASS \citep{AMASS:ICCV:2019} and MotionHub\citep{ling2024motionllama} datasets. The inner ring represents high-level motion domains, while the outer ring shows their corresponding subcategories. Segment sizes reflect the relative number of motion classes within each subgroup}
    \label{fig:composition_dataset}
\end{figure}

\newpage

\section{Additional examples of motion generation using our method}

\begin{figure*}[htbp]
\centering
\begin{tabular}{@{}p{\textwidth}@{}}
\centering
\includegraphics[width=\linewidth,keepaspectratio]{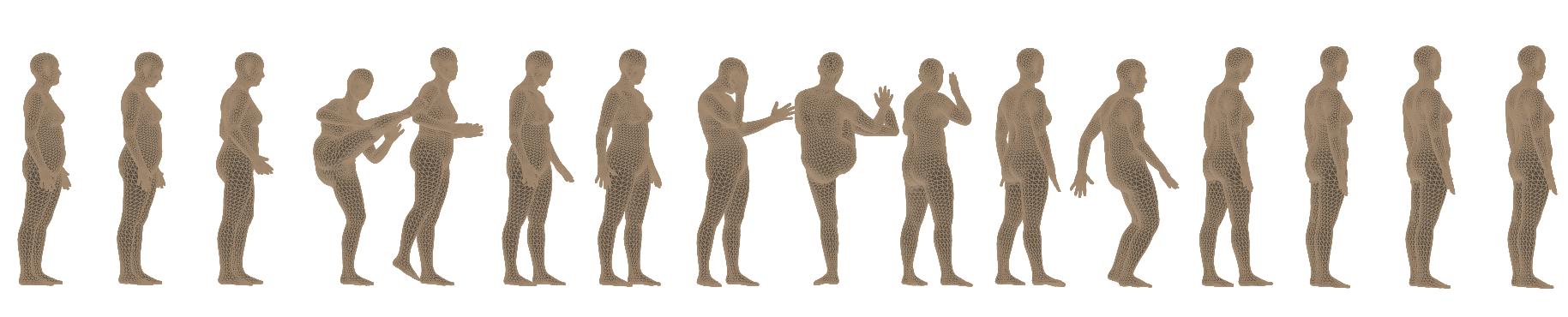}
\tabularnewline
\midrule
\raggedright
A person gets into a stance and kicks with his right leg. Then the person does a kick spin to the left, then kicks slowly with their left leg.
\tabularnewline

\centering
\includegraphics[width=\linewidth,keepaspectratio]{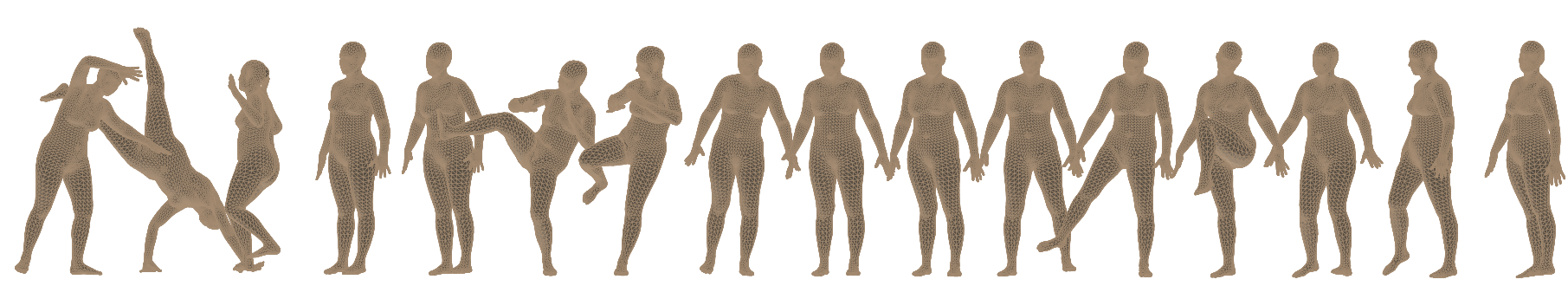}
\tabularnewline
\midrule
\raggedright
A person performing a clean cartwheel. Then a man bends his right leg then kicks it out in the air. Then the person kicking/swinging legs from side to side.
\tabularnewline

\centering
\includegraphics[width=\linewidth,keepaspectratio]{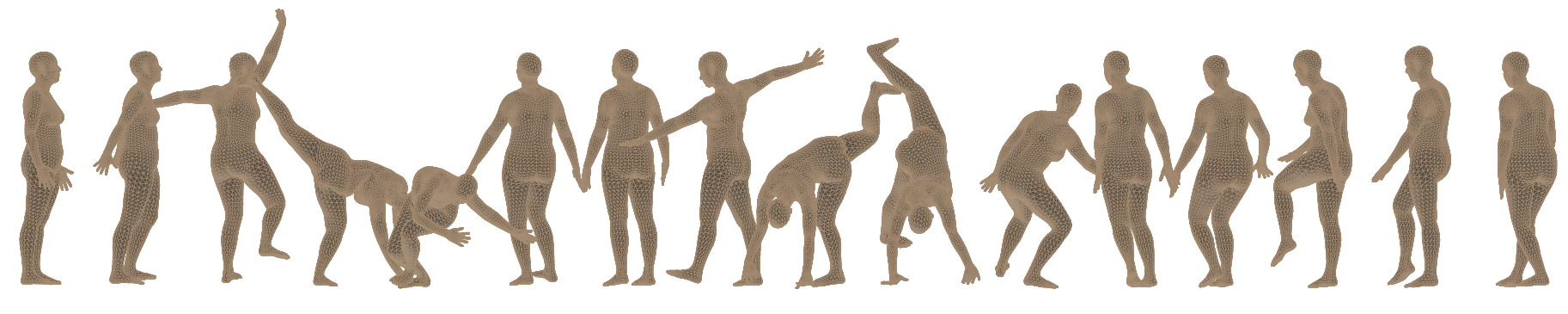}
\tabularnewline
\midrule
\raggedright
A standing person performs a cartwheel starting with their right hand, then a person is performing a cartwheels. Then the figure tilts and then kicks the air above knee level with its left leg, it then does a lower kick with its right leg.
\tabularnewline

\centering
\includegraphics[width=\linewidth,keepaspectratio]{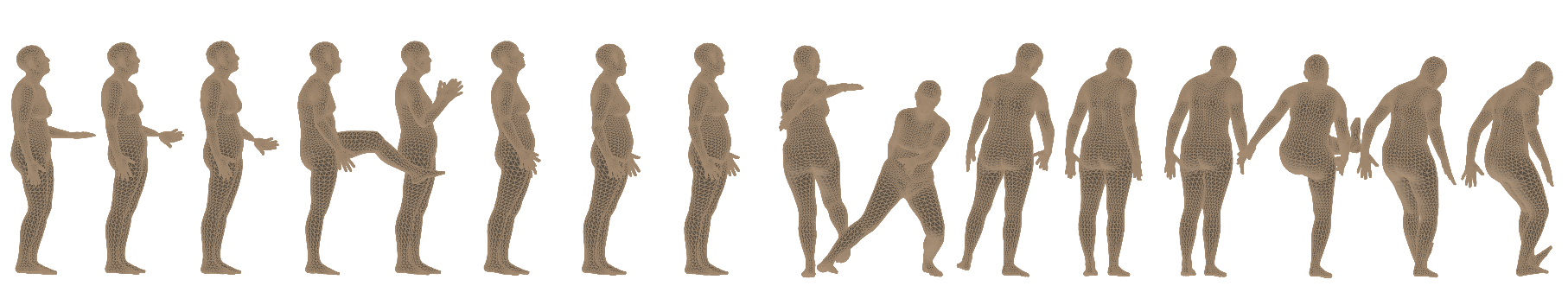}
\tabularnewline
\midrule
\raggedright
A person throws something with their left hand then kicks with their left foot before catching the object with both hands then a man jumps then kicks the air whilst moving to the opposite end of the room. Then the person takes a few steps forward kicks something with their left foot.
\tabularnewline
\end{tabular}
\end{figure*}



\end{document}